\documentclass{article}
\usepackage{arxiv}

\usepackage[utf8]{inputenc} 
\usepackage[T1]{fontenc}    
\usepackage[square,numbers]{natbib}
\usepackage{hyperref}       
\usepackage{url}            
\usepackage{booktabs}       
\usepackage{amsmath}
\usepackage{amsfonts}       
\usepackage{nicefrac}       
\usepackage{microtype}      

\usepackage{lipsum}         
\usepackage{graphicx}

\usepackage{doi}

\usepackage[table,xcdraw]{xcolor}
\usepackage[english]{babel}

\usepackage{cleveref}       

\usepackage{float}
\usepackage{amsthm}
\usepackage{mathrsfs}
\usepackage{amssymb}
\usepackage{pgfplots}
\usepackage{booktabs}
\usepackage{hhline}

\usepackage{doi}
\usetikzlibrary{spy}
\usepackage{tikz,pgfplots,pgf}
\usetikzlibrary{arrows,shapes,calc}
\usetikzlibrary{positioning}
\usepackage{neuralnetwork}
\usepackage[normalem]{ulem}

\newtheorem{example}{Example}
\newtheorem{theorem}{Theorem}
\newtheorem{proposition}{Proposition}
\newtheorem{definition}{Definition}

\newtheorem{lemma}{Lemma}
\newtheorem{notation}{Notation}

\pgfplotsset{compat=1.13}

\usepackage{xcolor, soul}
\sethlcolor{yellow}

\author{Ismaïl Baaj \\ Univ. Artois, CNRS, CRIL, F-62300 Lens, France \\  \href{baaj@cril.fr}{baaj@cril.fr}}

\title{Maximal Consistent Subsystems of Max-T Fuzzy Relational Equations}

\hypersetup{
pdftitle={Maximal Consistent Subsystems of Max-T Fuzzy Relational Equations},
pdfsubject={cs.AI},
pdfauthor={Ismaïl Baaj},
pdfkeywords={Fuzzy set theory ; Systems of fuzzy relational equations},
}
\date{}

\begin{document}

\maketitle

\begin{abstract}
  In this article, we study the inconsistency of a  system of $\max-T$ fuzzy relational equations of the form $A \Box_{T}^{\max} x = b$, where $T$ is a t-norm among $\min$, the product or Lukasiewicz's t-norm. For an inconsistent $\max-T$ system, we  directly construct  a canonical  maximal   consistent subsystem (w.r.t  the inclusion order). The main tool used to obtain it is  the analytical formula which compute       the Chebyshev distance  $\Delta = \inf_{c \in \mathcal{C}} \Vert b - c \Vert$ associated to the inconsistent $\max-T$ system, where  $\mathcal{C}$ is the set of second members of consistent systems defined with the same matrix $A$. Based on the same analytical formula, we give, for an inconsistent $\max-\min$ system, an efficient method to obtain all its consistent subsystems,  and we show how to iteratively get all its maximal consistent subsystems.
\end{abstract}


\section{Introduction}

Sanchez's seminal work on systems of $\max-\min$  fuzzy relational equations established necessary and sufficient conditions for a system to be consistent, i.e., to have solutions \cite{sanchez1976resolution}. In \cite{sanchez1977}, Sanchez showed that if the
system is consistent, it has a greater solution and a finite set of minimal solutions, and he then described the complete set of
solutions of the system.  His work was then followed by studies on solving systems based on $\max-T$ composition \cite{di1984fuzzy,miyakoshi1985solutions,pedrycz1982fuzzy,pedrycz1985generalized} where $T$ is a given t-norm. However, the inconsistency of $\max-T$ systems  remains difficult to address.

Some authors \cite{cechlarova1999resolving,liPhD2009fuzzy} highlighted that handling the inconsistency of a  system of $\max-T$ fuzzy relational equations can be tackled by finding its maximal consistent subsystems. Formally, given an inconsistent system defined by a set of $n$ equations whose indexes are $1, 2, \cdots, n$, a consistent subsystem defined by a subset of these equations, whose indexes form a subset $R \subseteq \{1,2, \cdots,n\}$, is maximal if any subsystem defined by a strict superset of $R$ is inconsistent.

In this article, using   recent works \cite{baaj2023chebyshev,baaj2023maxmin} on the handling of the  inconsistency of  $\max-T$ systems of the form $A \Box_{T}^{\max} x = b$,  where  $T$ is a t-norm among $\min$, product or Łukasiewicz's t-norm, we   directly construct  a canonical maximal consistent subsystem of the system $A \Box_{T}^{\max} x = b$. The indexes of the equations composing this maximal consistent subsystem are obtained by computing, using the $L_\infty$ norm, the Chebyshev distance $\Delta = \inf_{c \in \mathcal{C}} \Vert b - c \Vert$, where $\mathcal{C}$ is the set of second members of consistent systems defined with the same matrix $A$. The author of \cite{baaj2023chebyshev,baaj2023maxmin} gave  three analytical formulas for computing the Chebyshev distances $\Delta$ associated to the three $\max-T$ systems. \\
For this purpose, we start by showing that these three analytic formulas (which we remind in the article, see (\ref{eq:Deltamaxmin}) for $\max-\min$ system, (\ref{eq:deltap}) for $\max-$product system and (\ref{eq:DeltaL}) for max-Łukasiewicz system) 
have a canonical single expression,    which depends only on the t-norm $T$ considered.
For a $\max-T$ system whose matrix is of size $(n,m)$,   $\Delta$'s canonical form is expressed as $\Delta = \max_{1 \leq i \leq n}  \delta_i$ where $\delta_i = \min_{1 \leq j \leq m} \max_{1 \leq k \leq n} \delta_{ijk}^T$, see (\ref{eq:deltaijkT}). The computation of (at most)  $n^2 \cdot m$ numbers $\delta_{ijk}^T$ is therefore necessary to obtain the Chebyshev distance. From $\Delta$'s canonical form, we introduce the set $N_c$, see (\ref{eq:nc}), composed of the indexes $i$ of the equations of the considered system, whose corresponding $\delta_i$ is equal to zero. In order to prove that the subsystem defined by the set $N_c$ is a maximal consistent subsystem,   we first give  an important characterization of  the case $\delta_{ijk}^T > 0$, where $\delta_{ijk}^T$ is involved in   $\Delta$'s canonical form (see (Lemma \ref{Lemma:2}) and also (Lemma \ref{lemma:biaijplus})).
Then, we study any equation whose index $i$ is in the intersection of the complement of the set $N_c$ and a set   $R \subseteq \{1,2,\dots,n\}$ whose corresponding subsystem is consistent,   see (Lemma \ref{lemma:4}), (Lemma \ref{lemma:5}) and (Lemma \ref{Lemma:6}). These last three lemmas and (Proposition \ref{prop:taree}) let us prove our main result in (Theorem \ref{th:1}): \textit{the subsystem formed by the equations of the system whose indexes are in the set $N_c$ is a maximal consistent subsystem.} Since the Chebyshev distance $\Delta$ requires the computation of (at most)  $n^2\cdot m$ numbers, the construction of this maximal consistent subsystem  has the same computational complexity.
\\
We then  study the set formed by the consistent subsystems of an inconsistent $\max-\min$ system. For this purpose, we first arrange in ascending order the coefficients of the second member of the system, and we rely on the Chebyshev distance associated to any subsystem $(S_R)$ defined by a set  $R \subseteq \{1,2,\dots,n\}$, see (Lemma \ref{Lemma:subsys}). Given a  consistent subsystem $(S_R)$ of the subsystem $(S_{\{1 , 2 , \dots s\}})$  where $1 \leq s\leq n-1$, we can construct, by computing $n - s$ explicit numbers, all the consistent subsystems which are defined by subsets $R' \subseteq \{1,2,\dots,n\}$ such that  $R \subset R'$ and $\text{card}(R') = \text{card}(R) + 1$. 
Denoting by  ${\cal E}^s$ 
the set formed by the consistent subsystems of the subsystem    $(S_{\{1 , 2 , \dots s\}})$, 
we thus obtain a method to build ${\cal E}^{s + 1}$  (related to $(S_{\{1 , 2 , \dots s+1\}})$) from ${\cal E}^s$, see (\ref{eq:esplus1toes}), which can be used to get all the maximal consistent subsystems of the whole $\max-\min$ system, see (Proposition \ref{prop:mcsprop}). 

The article is structured as follows. In (Section \ref{sec:background}), we remind the necessary background for solving  $\max-T$ systems and present some of the results on the inconsistency of these systems proven in \cite{baaj2023chebyshev,baaj2023maxmin} i.e.,  the formulas to compute the Chebyshev distances $\Delta$ for a $\max-\min$ system, a $\max-$product system and a max-Łukasiewicz system, see  (\ref{eq:Deltamaxmin}), (\ref{eq:deltap}) and (\ref{eq:DeltaL}) respectively. In (Section \ref{sec:preliminaries}), we   obtain a canonical single expression for these three formulas,  which depends only on the t-norm considered, see (\ref{eq:deltacanonique}).  
 We introduce some notations (Notation \ref{notations:prelem}) and the set $N_c$, see (\ref{eq:nc}). Then, we characterize the  case $\delta^T_{ijk} > 0$, where $\delta^T_{ijk}$ is  involved in $\Delta$'s canonical form, and we provide a formula for the Chebyshev distance associated to a subsystem of the system, which is defined by a subset of equations of the considered system. In (Section \ref{sec:maximal}), we prove some preliminary results 
((Lemma \ref{lemma:4}), (Lemma \ref{lemma:5}), (Lemma \ref{Lemma:6}) and (Proposition \ref{prop:taree})), which are necessary to establish  (Theorem \ref{th:1}).
In (Section \ref{sec:findingminmaxCS}), we present our method for efficiently finding all consistent subsystems of an inconsistent max-min system and apply it to obtain all the maximal consistent systems of the inconsistent system. 
Finally, we conclude by giving some potential applications of our  results, in particular some based on $\max-T$ learning methods.

\section{Background}
\label{sec:background} 

\noindent In this section,  $T$ denotes a continuous t-norm and $\mathcal{I}_T$ its associated residual implicator  \cite{klement2013triangular}. We remind  the three main t-norms ($\min$, product and Łukasiewicz's t-norm) and their associated residual implicator (the Gödel implication, the Goguen implication and Łukasiewicz's implication, respectively). \\ 
We remind the necessary background for solving systems of $\max-T$ fuzzy relational equations of the form $A \Box_{T}^{\max} x = b$. Finally, we present some recent results of \cite{baaj2023chebyshev,baaj2023maxmin} on the handling of the inconsistency of systems of $\max-T$ fuzzy relational equations, in which the author gave, for each $\max-T$ system, an explicit analytic formula for computing the Chebyshev distance $\Delta = \inf_{c \in \mathcal{C}} \Vert b - c \Vert$, where $\mathcal{C}$ is the set of second members of consistent systems defined with the same matrix $A$. 

The following notations are reused from \cite{baaj2023chebyshev,baaj2023maxmin}:
\begin{notation}
   The set $[0,1]^{n\times m}$ denotes the set of matrices of size $(n,m)$ i.e., $n$ rows and $m$ columns, whose components are in $[0,1]$.  The set $[0,1]^{n\times 1}$ is the set of column vectors of $n$ components and $[0,1]^{1\times m}$ is  the set of row matrices of $m$ components.

\noindent The order relation $\leq$ on the set $[0,1]^{n\times m}$ is defined by:
\[ A \leq B \quad \text{iff we have} \quad  a_{ij} \leq b_{ij} \quad \text{ for all } \quad 1 \leq i \leq n, 1 \leq j \leq m,    \]
\noindent where $A=[a_{ij}]_{1 \leq i \leq n, 1 \leq j \leq m}$ and $B=[b_{ij}]_{1 \leq i \leq n, 1 \leq j \leq m}$. 

 For $x,y,z,u,\delta \in [0,1]$, we put:
\begin{itemize}
    \item $x^+ = \max(x,0)$,
    \item $\overline{z}(\delta) = \min(z+\delta,1)$, 
    \item $\underline{z}(\delta) = \max(z-\delta,0) = (z-\delta)^+$.
\end{itemize}
\end{notation}

\subsection{T-norms and their associated residual implicators }

\noindent
A triangular-norm  (t-norm, see  \cite{klement2013triangular}) is a map  $T: [0,1] \times [0,1] \mapsto [0,1]$, which satisfies:
\begin{itemize}
    \item[] $T$ is commutative:   $T(x,y) = T(y,x)$, 
    \item[] $T$ is associative:  $T(x, T(y, z)) = T(T(x, y), z)$, 
    \item[] $T$ is increasing  : $x \leq x' \quad \text{and} \quad y \leq y' \, \Longrightarrow \, T(x, y) \leq T(x', y')$,
    \item[]$T$ has $1$ as neutral element: $T(x, 1) = x$.
\end{itemize}

\noindent To the   t-norm $T$ is associated   the residual implicator  
${\cal I}_T : [0, 1 ] \times [0, 1] \rightarrow [0, 1 ] : (x, y) \mapsto {\cal I}_T(x, y) = \sup\{z \in [0, 1]\,\mid\, T(x, z) \leq y\}$.

\noindent
 For all $a, b   \in [0, 1]$, the main properties of the residual implicator ${\cal I}_T$ associated to a continuous t-norm $T$ are: 
\begin{itemize}
    \item   
${\cal I}_T(a, b) = \max\{z \in [0, 1]\,\mid\, T(a, z) \leq b\}$. Therefore, $T(a, {\cal I}_T(a, b)) \leq b$.
\item
  ${\cal I}_T$ is left-continuous and decreasing in its first argument as well as right-continuous and increasing in its second argument.
  \item
  For all $z\in [0, 1]$, we  have: 
  $$T(a, z) \leq b \Longleftrightarrow z \leq {\cal I}_T(a, b).$$
  \item We have 
  $\,b \leq {\cal I}_T(a, T(a, b))$.
\end{itemize}

\noindent The t-norm $\min$ denoted by $T_M$,  has a residual implicator $\mathcal{I}_{T_M}$ which is the Gödel implication:
\begin{equation}\label{eq:tnormmin}
    T_M(x,y) = \min(x, y) \quad ;  \quad {\cal I}_{T_M}(x,y) = x \underset{G}{\longrightarrow} y = \begin{cases}1 & \text{ if } x \leq y \\ \ y &\text{ if } x > y \end{cases}.
\end{equation}

The t-norm defined by the usual product   is denoted by $T_P$. Its associated residual implicator is the Goguen implication:
\begin{equation}\label{eq:tnormproduct}
    T_P(x,y) = x \cdot y \quad ;  \quad {\cal I}_{T_P}(x,y) = x \underset{GG}{\longrightarrow} y = \begin{cases}1 & \text{ if } x \leq y \\ \frac{y}{x} &\text{ if } x > y \end{cases}.
\end{equation}
Łukasiewicz's t-norm is denoted by $T_L$ and its associated residual implicator is Łukasiewicz's implication ${\cal I}_{T_L}$:
\begin{equation}\label{eq:tnormluka}
    T_{L}(x,y)= \max(x + y - 1, 0) = {(x + y - 1)}^{+} \quad ; \quad 
 {\cal I}_{T_L}(x,y)  = x \underset{L}{\longrightarrow} y = \min(1-x+y,1).
\end{equation}

\subsection{Solving systems of \texorpdfstring{$\max-T$}{max-T} fuzzy relational equations}

A system of $\max-T$ fuzzy relational equations based on a matrix $A=[a_{ij}] \in [0,1]^{n\times m}$   and a column-vector $b=[b_{i}] \in [0,1]^{n\times 1}$   is of the form:\begin{equation}\label{eq:sys}
    (S): A \Box_{T}^{\max} x = b,
\end{equation}
\noindent where $x = [x_j]_{1 \leq j \leq m} \in [0,1]^{m\times 1}$ is  an unknown vector and the matrix product $\Box_{T}^{\max}$ uses the continuous  t-norm $T$ as the product and $\max$  as the addition.

\noindent Using the vector \begin{equation}\label{eq:gr:sol}
    e = A^t \Box_{{\cal I}_T}^{\min} b,
\end{equation}
\noindent where $A^t$ is the transpose of $A$ and the matrix product $\Box_{{\cal I}_T}^{\min}$ uses the residual implicator  ${\cal I}_T$ (associated to $T$) as the product and $\min$ as the addition, we have the following equivalence  proved by  Sanchez for $\max-\min$ composition \cite{sanchez1976resolution}, and extended to $\max-T$ composition by Pedrycz \cite{pedrycz1982fuzzy,pedrycz1985generalized} and Miyakoshi and Shimbo \cite{miyakoshi1985solutions}:
\begin{equation}\label{eq:consiste}
    A \Box_T^{\max} x = b \text{ is consistent}\Longleftrightarrow A \Box_T^{\max} e = b. 
\end{equation}

\begin{example}\label{ex:1}
We study the following $\max-\min$ system denoted $A \Box_{\min}^{\max} x = b$ where: 
\begin{equation}\label{eq:Abofpedrycz}
A = \begin{bmatrix}
    1 & 0.4 & 0.5 & 0.7\\
    0.7 & 0.5 & 0.3 & 0.5\\
    0.2 & 1 & 1 & 0.6\\ 
    0.4 & 0.5 & 0.5 & 0.8
\end{bmatrix}\text{ and } b = \begin{bmatrix}
        0.8 \\ 0.7 \\ 0.4 \\ 0.4
    \end{bmatrix}.
\end{equation}
We compute the potential greatest solution of the system  using the Gödel implication $\rightarrow_G$, see (\ref{eq:tnormmin}), which is associated to the t-norm $\min$: $$ e = A^t \Box_{\rightarrow_G}^{\min} b = \begin{bmatrix}
    0.8\\ 
    0.4 \\
    0.4 \\
    0.4
\end{bmatrix}.$$
\noindent We check that: $$A \Box_{\min}^{\max} e = \begin{bmatrix}
        0.8 \\ 0.7 \\ 0.4 \\ 0.4
    \end{bmatrix} = b.$$ So the system $A \Box_{\min}^{\max} x = b$ is consistent. 

\end{example}

\subsection{Chebyshev distance associated to the second member of a system of \texorpdfstring{$\max-T$}{max-T} fuzzy relational equations}

\noindent To the matrix $A$ and the second member $b$ of the system $(S)$ of $\max-T$ fuzzy relational equations, see (\ref{eq:sys}), is  associated  the set of  vectors $c = [c_i] \in [0,1]^{n \times 1}$ such that the system $A \Box_{T}^{\max} x = c$ is consistent:
\begin{equation}\label{def:setofsecondmembersB}
    \mathcal{C} = \{ c = [c_i] \in {[0,1]}^{n \times 1} \mid  A \Box_{T}^{\max} x = c \text{ is consistent} \}.
\end{equation}
\noindent This set allows us to define the Chebyshev distance associated to the second member $b$ of the system $(S)$.
 \begin{definition}\label{def:chebyshevdist}
The Chebyshev distance associated to the second member $b$ of the system $(S): A \Box_{T}^{\max}x = b$ is: 
\begin{equation}\label{eq:delta}
    \Delta = \Delta(A,b) =  \inf_{c \in \mathcal{C}} \Vert b - c \Vert 
    \end{equation}

    \end{definition}
\noindent where:
\[ \Vert b - c \Vert = \max_{1 \leq i \leq n}\mid b_i - c_i\mid.\]

The following result was proven for $\max-\min$ system in \cite{cuninghame1995residuation} and recently extended to $\max-T$ systems in \cite{baaj2023chebyshev}:

\begin{equation}\label{eq:deltaCUNINGH}
    \Delta =        \min\{\delta\in [0, 1] \mid \underline b(\delta) \leq F(\overline b(\delta))\},
\end{equation}

\noindent which involves a vector inequality based on the following application:
\begin{equation}\label{eq:F} 
    F :  [0, 1]^{n \times 1} \longrightarrow [0, 1]^{n \times 1} : c =[c_i] \mapsto F(c) = A \Box_{T}^{\max} (A^t \Box_{{\cal I}_T}^{\min}
    c) = [F(c)_i]
\end{equation}
\noindent where:
\begin{equation}\forall i \in \{1, 2, \dots, n\},\, F(c)_i = \max_{1 \leq j \leq m} T(a_{ij}, \min_{1 \leq k \leq n} 
{\cal I}_T(a_{kj}, c_k)).\end{equation}
We remind (see \cite{baaj2023chebyshev,baaj2023maxmin}) that for any $c\in [0 , 1]^n$ , we have the inequalities: 
\begin{equation}\label{Plus2}
\forall i \in \{1 , 2 , \dots , n\}\,,\, F(c)_i \leq c_i.
\end{equation}

\noindent By solving (\ref{eq:deltaCUNINGH}) in the case of a system of $\max-\min$ fuzzy relational equations $A \Box_{\min}^{\max}x = b$, the author of \cite{baaj2023maxmin} gave the following analytical formula for computing the Chebyshev distance associated to its second member $b$: 
\begin{equation} \label{eq:Deltamaxmin}
\Delta = \max_{1 \leq i \leq n}  \min_{1 \leq j \leq m}\,\max[ (b_i - a_{ij})^+,  \max_{1 \leq k \leq n}\,  \,\sigma_G\,(b_i, a_{kj}, b_k)],\end{equation}
\noindent \text{ where }
\begin{equation}\label{eq:sigmaG}
    \sigma_G(x, y, z) = \min( \frac{(x - z)^+}{2},(y - z)^+).
\end{equation}

Similarly, for the case of a system of $\max$-product fuzzy relational equations $A \Box_{T_P}^{\max}x = b$,  the author of \cite{baaj2023chebyshev} gave the  following analytical formula for computing the Chebyshev distance associated to $b$:
\begin{equation} \label{eq:deltap}
\Delta = \max_{1 \leq i \leq n} \min_{1 \leq j \leq m}\, \max_{1 \leq k \leq n}\,\sigma_{GG}\,(a_{ij}, b_i, a_{kj}, b_k),     
\end{equation}
\noindent where
\begin{equation}\label{eq:sigmagg}
    \sigma_{GG}(u,x,y,z) = \max[(x-u)^+, \min(\varphi(u,x,y,z),(y-z)^+)] \end{equation}
    \noindent \text{ and } \begin{equation}\label{eq:varphi}
    \varphi(u,x,y,z) = \begin{cases}\frac{(x \cdot y - u\cdot z)^+}{u+y} &\text{ if } u > 0 \\ x &\text{ if } u = 0\end{cases}. \end{equation}

The author of \cite{baaj2023chebyshev} also gave the following analytical formula for computing Chebyshev distance associated to the second member $b$ of a system of $\max$-Łukasiewicz fuzzy relational equations $A \Box_{T_L}^{\max}x = b$:
\begin{equation} \label{eq:DeltaL}
\Delta  = \max_{1 \leq i \leq n}  \min_{1 \leq j \leq m}\, \max_{1 \leq k \leq n}\,\sigma_{L }\,(1  - a_{ij}, b_i, a_{kj}, b_k),
\end{equation}
\noindent where
\begin{equation}\label{eq:sigmaL}
  \sigma_L(u,x,y,z) = \min(x, \max(v^+, \frac{(v + y - z )^+}{2})) \text{ with } v = x+u-1.
\end{equation}

From these formulas, the author of \cite{baaj2023chebyshev,baaj2023maxmin} proved that $F(\overline{b}(\Delta))$ is the greatest Chebyshev approximation  of the second member $b$, so the system $A \Box_{T}^{\max} x = F(\overline{b}(\Delta))$ is consistent and $\Vert F(\overline{b}(\Delta)) - b \Vert = \Delta$. Therefore, the author showed the following important equivalence for $\max-T$ systems:
\begin{equation}\label{eq:deltaegalzero}
    \Delta = 0 \Longleftrightarrow \text{ The system }(S) \text{ is consistent.}
\end{equation}

\begin{example}(continued) 
We reuse the matrix $A$ and the vector $b$ of (Example \ref{ex:1}), see (\ref{eq:Abofpedrycz}).

\begin{itemize}
    \item The Chebyshev distance associated to the $\max-\min$ system $A \Box_{\min}^{\max} x =b$ is equal to zero  (this system is consistent).

    \item  The Chebyshev distance associated to the second member of the $\max-$product system $A \Box_{T_P}^{\max} x =b$ is roughly equal to $0.083$ (this system is inconsistent).  We compute $F(\overline{b}(\Delta)) = \begin{bmatrix}
        0.883\\ 0.6181\\ 0.483\\ 0.483
    \end{bmatrix}$ and the system $A \Box_{T_P}^{\max} x =F(\overline{b}(\Delta))$ is consistent. 
    \item The Chebyshev distance associated to the  second member of the $\max-$Łukasiewicz system $A \Box_{T_L}^{\max} x =b$ is equal to $0.1$ (this system is inconsistent).  We compute $F(\overline{b}(\Delta)) = \begin{bmatrix}
        0.9\\ 0.6\\ 0.5\\ 0.5 
    \end{bmatrix}$ and the system $A \Box_{T_L}^{\max} x =F(\overline{b}(\Delta))$ is consistent. 
\end{itemize}

\end{example}

\section{Preliminaries}
\label{sec:preliminaries}
In this section, we begin by showing that the formulas ((\ref{eq:Deltamaxmin}), (\ref{eq:deltap}) and (\ref{eq:DeltaL})) of the Chebyshev distance $\Delta = \Delta(A,b)$  associated to the three $\max-T$ systems of the form $(S) : A \Box_T^{\max} x = b$, see (\ref{eq:sys}),   have a canonical single expression, which depends only on the t-norm $T$. 
We then give some notations (Notation \ref{notations:prelem}). In particular, we denote  by $(S_R)$  the subsystem of the system $(S)$ defined using    the set $R \subseteq \{1,2,\dots,n\}$ of the indexes of the equations taken from the system $(S)$ to form the subsystem $(S_R)$.

\noindent We establish some properties ((Lemma \ref{lemma:biaijplus}), (Lemma \ref{Lemma:2}) and (Lemma \ref{Lemma:subsys})) of the numbers $\delta_{ijk}^T$ involved in the canonical formula of $\Delta$. These results will be useful in the next section to prove (Theorem \ref{th:1}), which allows us to obtain a canonical maximal consistent subsystem of an inconsistent system $(S)$ just by computing the Chebyshev distance $\Delta$ associated to the system $(S)$.

\subsection{\texorpdfstring{$\Delta$'s canonical form and notations}{Delta's canonical form and notations}}
Let us rewrite the Chebyshev distance associated to the second member of a system of $\max-\min$ fuzzy relational equations, see (\ref{eq:Deltamaxmin}), as follows:
\begin{equation} \label{eq:Deltamaxmin2}
 \Delta = \max_{1 \leq i \leq n}  \min_{1 \leq j \leq m}\, \max_{1 \leq k \leq n} \max[(b_i - a_{ij})^+,\,  \,\sigma_G\,(b_i, a_{kj}, b_k)]\end{equation}

Then, from the formulas of the Chebyshev distances $\Delta$ associated to each of three systems of $\max-T$ fuzzy relational equations, see  (\ref{eq:deltap}), (\ref{eq:DeltaL}) and (\ref{eq:Deltamaxmin2}), we can give  a canonical formula for $\Delta$ which only depends on the choice of the t-norm $T$:
\begin{equation}\label{eq:deltacanonique}
   \Delta = \max_{1 \leq i \leq n}  \delta_i  \quad \text{ with }
    \quad \delta_i = \min_{1 \leq j \leq m}\, \delta(i , j) \quad \text{ and }
    \quad  \delta(i , j)=\max_{1 \leq k \leq n} \delta_{ijk}^T
\end{equation}
\noindent where:
\begin{equation}\label{eq:deltaijkT}
    \delta_{ijk}^T = \begin{cases}
        \max[(b_i - a_{ij})^+,\,  \,\sigma_G\,(b_i, a_{kj}, b_k)] & \text{ if } T = T_M  \text{ (min)} \\
        \sigma_{GG}\,(a_{ij}, b_i, a_{kj}, b_k)  & \text{ if } T = T_P  \text{ (product)} \\
        \sigma_{L }\,(1  - a_{ij}, b_i, a_{kj}, b_k) & \text{ if } T = T_P \text{ (Łukasiewicz's t-norm)}
    \end{cases}.
\end{equation}
(We remind that $\sigma_G$, $\sigma_{GG}$ and $\sigma_L$ are defined in (\ref{eq:sigmaG}), (\ref{eq:sigmagg}) and (\ref{eq:sigmaL}) respectively and were introduced in  \cite{baaj2023chebyshev,baaj2023maxmin}). 

In this form, we can see that  $\Delta$ requires  computing (at most) $n^2\cdot m$ numbers $\delta_{ijk}^T$.

For a  $\max-T$ system $A \Box_{T}^{\max} x = b$, see (\ref{eq:sys}), we use the following notations:
\begin{notation}\label{notations:prelem} \mbox{}
\begin{itemize}
    \item $N = \{1 ,2 , \dots , n\}, \quad M = \{1 ,2 , \dots , m\}$,
    \item for any subset $R \subseteq \{1,2,\dots, n \}$, we form the following $\max-T$ subsystem:
    \[ (S_R) : A_R \Box_T^{\max} x =  b_R, \text{ where }
A_R =  [a_{ij}]_{i\in R , 1 \leq j \leq m} \text{ and }   b_R = [b_{i}]_{i\in R},
 \] 
 \noindent whose associated Chebyshev distance is denoted $\Delta_R := \Delta(A_R,b_R)$.\\

Thus, for any $i\in\{1 , 2 , \dots , n\}$ , the $\max-T$ system 
$(S_{\{i\}}) : A_{\{i\}} \Box_T^{\max} x = b_{\{i\}}$ is the system reduced to the $i$-th equation of $(S)$. 
\end{itemize}
\end{notation}
\noindent From the canonical form of $\Delta$, see (\ref{eq:deltacanonique}), we extract the indexes of equations whose corresponding $\delta_i$ is equal to zero, in order to form the following set:

\begin{equation}\label{eq:nc}
    N_c = \bigg\{ i \in \{1 ,2 , \dots , n\}\,\mid\,  \delta_i = 0\bigg\}. 
\end{equation}

\noindent The complement of the set $N_c$ is denoted $\overline{N_c}$.

\begin{example}(continued) 
We reuse the matrix $A$ and the vector $b$ of (Example \ref{ex:1}), see (\ref{eq:Abofpedrycz}).

\begin{itemize}
    \item The $\max-\min$ system $A \Box_{\min}^{\max} x =b$ is consistent. We have $\delta_1 = \delta_2 = \delta_3 = \delta_4 = 0$ and the set $N_c$ is equal to $\{1,2,3,4\}$.

    \item  The $\max-$product system $A \Box_{T_P}^{\max} x =b$ is inconsistent.   We have $\delta_1 = \delta_3 = \delta_4  = 0$ and $\delta_2 = 0.083$, so the set $N_c$ is equal to $\{1,3,4\}$. 
    \item The $\max-$Łukasiewicz system $A \Box_{T_L}^{\max} x =b$ is inconsistent.  We have $\delta_1 = \delta_3 = \delta_4  = 0$ and $\delta_2 = 0.1$, so the set $N_c$ is equal to $\{1,3,4\}$.
\end{itemize}
\end{example}

\subsection{Preliminary results}

From $\Delta$'s canonical form, see (\ref{eq:deltacanonique}), we establish:
\begin{lemma}\label{lemma:biaijplus}
For all $i\in N$ and $ j\in M$ (Notation \ref{notations:prelem}), we have $\delta^T_{iji} = (b_i - a_{ij})^+$. 
\end{lemma}
\begin{proof}\mbox{}
\begin{itemize}
    \item For $T = T_M$ (min), from $\sigma_G$, see (\ref{eq:sigmaG}), we have for all $i\in N$ and $ j\in M$:  $$\sigma_G(b_i,a_{ij},b_i) = \min(\frac{(b_i - b_i)^+}{2}, (a_{ij}-b_i)^+) = 0.$$ Therefore, from (\ref{eq:deltaijkT}), we have: 
    \begin{align*}
        \delta_{iji}^{T_M} &= \max[(b_i - a_{ij})^+,\,  \,\sigma_G\,(b_i, a_{ij}, b_i)] \\ &
        = (b_i - a_{ij})^+.
    \end{align*}
    \item For $T = T_P$ (product), from the formula of $\varphi$, see (\ref{eq:varphi}), we have for all $i\in N$ and $ j\in M$:  
    $$\varphi(a_{ij},b_i,a_{ij},b_i) = \begin{cases} 0 &\text{ if } a_{ij} > 0 \\ b_i & \text{ if } a_{ij} = 0  \end{cases}.$$
    We remark that if $a_{ij} = 0$, then $(a_{ij}-b_i)^+ = 0$.
    \noindent Thus, whatever if $a_{ij} > 0$ or $a_{ij} = 0$, from the formula of $\sigma_{GG}$, see (\ref{eq:sigmagg}), we have:
    \begin{align*}
        \delta_{iji}^{T_P} &= \sigma_{GG}(a_{ij},b_i,a_{ij},b_i)\\
        &= \max[(b_i-a_{ij})^+, \min(\varphi(a_{ij},b_i,a_{ij},b_i),(a_{ij}-b_i)^+)]\\ 
        &= (b_i-a_{ij})^+.
    \end{align*}
    \item For $T = T_L$ (Łukasiewicz’s t-norm), from the formula of $\sigma_L$, see (\ref{eq:sigmaL}), we have:
    \begin{align*}
        \delta_{iji}^{T_L} &= \sigma_L(1 - a_{ij},b_i,a_{ij},b_i) \\
        &= \min(b_i, \max( (b_i - a_{ij})^+, \frac{(b_i - a_{ij} + a_{ij} - b_i)^+}{2}) \\
        &= \min(b_i,(b_i - a_{ij})^+) = (b_i - a_{ij})^+.
    \end{align*}

\end{itemize}
\end{proof}

\noindent We establish the following equivalences:  
\begin{lemma}\label{Lemma:2}
Let $i\in N$ and $j\in M$ and we assume that  $b_i \leq a_{ij}$. Then, for all $k\in N$, we have:
\begin{itemize}
    \item for $T = T_M$ (min): 
    \begin{equation}\label{eq:M1}
\delta^T_{ijk} > 0 \Longleftrightarrow
b_i > b_k \quad \text{and} \quad a_{kj} > b_k, 
\end{equation}
\item for $T = T_P$ (product): 
\begin{equation}\label{eq:P1}
\delta^T_{ijk} > 0 \Longleftrightarrow
a_{ij} > 0 \quad \text{ and } \quad a_{kj} > b_k 
\quad \text{ and } \quad \dfrac{b_i}{a_{ij}} > \dfrac{b_k}{a_{kj}},
\end{equation}
\item for $T = T_L$ (Łukasiewicz’s t-norm):
\begin{equation}\label{eq:L1}
\delta^T_{ijk} > 0 \Longleftrightarrow
b_i > 0 \quad \text{and} \quad b_i - a_{ij} > b_k - a_{kj}.
\end{equation}
\end{itemize}
\end{lemma}
\begin{proof}\mbox{}\\
For $T = T_M$, if $b_i \leq a_{ij}$ then $\delta^T_{ijk} = \sigma_G(b_i, a_{kj}, b_{k})$. From the formula of $\sigma_G$, see (\ref{eq:sigmaG}), we immediately deduce the equivalence. \\

\noindent For $T = T_P$, if $b_i \leq a_{ij}$ then $\delta^T_{ijk} = \sigma_{GG}(a_{ij}, b_i, a_{kj}, b_{k}) = \max[(b_i-a_{ij})^+, \min(\varphi(a_{ij}, b_i, a_{kj}, b_{k}),(a_{kj}-b_k)^+)] = \min(\varphi(a_{ij}, b_i, a_{kj}, b_{k}),(a_{kj}-b_k)^+)$. Thus, we have:
\[ \delta^T_{ijk} > 0 
\Longleftrightarrow 
 \varphi(a_{ij}, b_i, a_{kj}, b_{k}) > 0 \quad \text{and} \quad a_{kj} > b_k.\]

\noindent We  distinguish two cases when computing $\varphi(a_{ij}, b_i, a_{kj}, b_{k})$: it returns $\frac{(b_i \cdot a_{kj} - a_{ij} \cdot b_k)^+}{a_{ij} + a_{kj}}$  if $a_{ij} > 0$ or it returns $b_i$ if $a_{ij} = 0$ (see its definition in (\ref{eq:varphi})). To have $\varphi(a_{ij}, b_i, a_{kj}, b_{k}) > 0$, we must have $a_{ij} > 0$, as if $a_{ij} = 0$, we have $\varphi(a_{ij}, b_i, a_{kj}, b_{k})=b_i$ which would be equal to zero since we suppose $b_i \leq a_{ij}$. Finally, the inequality $\dfrac{b_i}{a_{ij}} > \dfrac{b_k}{a_{kj}}$ is equivalent to the strict positivity of 
  $\frac{(b_i \cdot a_{kj} - a_{ij} \cdot b_k)^+}{a_{ij} + a_{kj}}  $ (we have $a_{ij} + a_{kj} > 0$).

\noindent For $T = T_L$, we have 
$\delta^T_{ijk} = \sigma_{L }\,(1  - a_{ij}, b_i, a_{kj}, b_k) = \min(b_i, \max( (b_i - a_{ij})^+, \frac{(b_i - a_{ij} + a_{kj} - b_k)^+}{2})).$ 
Since $b_i \leq a_{ij}$, we have $\delta^T_{ijk} = \min(b_i, \frac{(b_i - a_{ij} + a_{kj} - b_k)^+}{2})$ and we immediately deduce the equivalence. 
\end{proof}

From (Notation \ref{notations:prelem}), we give the formula of the Chebyshev distance associated  to the subsystem $(S_R) : A_R \Box_T^{\max} x =  b_R$:
\begin{lemma}\label{Lemma:subsys}
For any non-empty subset $R \subseteq  \{1 , 2 , \dots , n\}$:
 \begin{enumerate}
\item   The Chebyshev distance associated to the subsystem $(S_R)$ is:
\begin{equation}\label{eq:deltaR} \Delta_R = \max_{i\in R} \min_{1 \leq j \leq m} \max_{k\in R} \delta^T_{ijk}.\end{equation}
\item We have the following equivalence:
\[(S_R) : A_R \Box_T^{\max} x =  b_R  \text{ consistent } \quad 
\Longleftrightarrow \quad 
\forall i\in R \, \exists j\in M \,\text{ such that }  
\max_{k\in R} \delta^T_{ijk}  = 0.\]
\end{enumerate}
\end{lemma}
The proof of this lemma is easy using (\ref{eq:deltaegalzero}) and (\ref{eq:deltacanonique}). \\
Let us illustrate this result with the subsystems reduced to a single equation:
\begin{example}
For all $i\in\{1 , 2 , \dots , n\}$, we have, see (Lemma \ref{lemma:biaijplus}):
\[ (S_{\{i\}}) : A_{\{i\}} \Box_T^{\max} x = b_{\{i\}} \text{ consistent } 
\Longleftrightarrow 
\exists j\in M \text{ such that } \delta^T_{iji}  = (b_i - a_{ij})^+ = 0. \]
\end{example}

\section{\texorpdfstring{Find a maximal consistent subsystem of an inconsistent system of $\max-T$ fuzzy relational equations}{Find a maximal consistent subsystem of an inconsistent system of max-t fuzzy relational equations}}
\label{sec:maximal}

Let us take a non-empty subset $R \subseteq N$ and assume that the subsystem $(S_R) : A_R \Box_T^{\max} x = b_R$ is consistent (Notation \ref{notations:prelem}). For an  index $i$ of an equation which is in the intersection of the set $R$ and the complement $\overline{N_c}$ of the set $N_c$, see (\ref{eq:nc}), we establish four important results ((Lemma \ref{lemma:4}), (Lemma \ref{lemma:5}),  (Lemma \ref{Lemma:6}) and (Proposition \ref{prop:taree})). 

\noindent Finally, we show our main result (Theorem \ref{th:1}), which allows us to obtain a canonical maximal consistent subsystem of an inconsistent system $(S)$ by computing the Chebyshev distance of the  system $(S)$.

\begin{lemma}\label{lemma:4}
Let $i\in R \cap\overline{N_c}$.  Then, we have:
\begin{equation}\label{eq:H1}\forall j \in M, \quad \max_{k\in N} \delta^T_{ijk}  > 0, \end{equation} 
\begin{equation}\label{eq:pasH1R}\exists    j \in M, \quad \max_{k\in R} \delta^T_{ijk}  = 0.\end{equation}
\end{lemma}
\begin{proof}
If $i\notin N_c$, it means that $\delta_i = 
\min_{j\in M} \max_{k\in N} \delta^T_{ijk} > 0$, from which we establish  (\ref{eq:H1}). From (\ref{eq:deltaR}), we also deduce (\ref{eq:pasH1R}).
\end{proof}
\begin{lemma}\label{lemma:5}
Let $i\in R \cap\overline{N_c}$ and $j\in M$ and we assume that: 
\begin{equation}\label{eq:nH1}\max_{k\in R} \delta^T_{ijk} = 0.\end{equation}

\noindent Then, we have:
\begin{enumerate}
\item $b_i \leq a_{ij}$.
\item There exists $k_1 \in \overline{R}$ such that $\delta^T_{ijk_1} > 0.$
\end{enumerate}
\end{lemma}
\begin{proof}
Let us take $k = i$ in (\ref{eq:nH1}), then we obtain $\delta^T_{iji} = 0$. From (Lemma \ref{lemma:biaijplus}), we deduce that we have:
\[ \delta^T_{iji}  = (b_i - a_{ij})^+ = 0.\]
So $b_i \leq a_{ij}$.

We apply (\ref{eq:H1}) to the index $j\in M$ and we obtain: 
\[ \max_{k\in N} \delta^T_{ijk} > 0. \]
By taking into account (\ref{eq:nH1}), we then deduce that there exists an index $k_1\in \overline{R}$ such that $\delta^T_{ijk_1} > 0.$
\end{proof}

\begin{lemma}\label{Lemma:6}
Let an index $i\in R \cap\overline{N_c}$ and an index $j\in M$ satisfy (\ref{eq:nH1}). For any  index   $ k_1\in \overline{R}  \cap\overline{N_c}$ such that $\delta^T_{ijk_1} > 0$, there exists an index $k_2 \in \overline{R}$ such that:  
\[ \delta^T_{ijk_1} > 0,  \quad \delta^T_{k_1jk_2} > 0, \quad   \delta^T_{ijk_2} > 0, \quad   k_1 \not= k_2. \] 
\end{lemma}
\begin{proof}
From (Lemma \ref{lemma:5}), we have $b_i \leq a_{ij}$.
An index $k_1\notin N_c$ means that  
$\forall l\in M,\,\, \max_{k\in N} \delta^T_{k_1 l k} > 0$. 

By taking $l = j$,  we obtain $\max_{k\in N} \delta^T_{k_1 j k} > 0.$

Let $k_2\in N$ be such that  $\delta^T_{k_1 j k_2} > 0$.
Let us check for each of the three t-norms that we have:
\[ k_1 \not= k_2, \quad \delta^T_{ijk_2} > 0, \quad  k_2 \in \overline{R}.  \]

\begin{itemize}
    \item For $T =T_M$ (min),  as we have $b_i \leq a_{ij}$, then by applying (\ref{eq:M1}) to $\delta^T_{ijk_1} > 0$, we deduce:
    \[ b_i > b_{k_1} \quad \text{and}\quad a_{k_1j} > b_{k_1}. \]
    The inequalities $a_{k_1j} > b_{k_1}$ and $\delta^T_{k_1jk_2} > 0$ allow us to apply (\ref{eq:M1}) and we obtain:
    \[ b_{k_1} > b_{k_2},  \quad   \quad a_{k_2 j} > b_{k_2}. \]
    So $k_1 \not=k_2$. As we now have:
    \[ b_i > b_{k_1} > b_{k_2},  \quad a_{k_1 j} > b_{k_1}, \quad a_{k_2 j} > b_{k_2}.\]
    and  also $b_i \leq a_{ij}$ , we deduce from 
    (\ref{eq:M1})
that we have $\delta^T_{ijk_2} > 0$.

As, by hypothesis, we have
$\max_{k\in R} \delta^T_{ijk} = 0$,  we necessarily have $k_2\in \overline{R}$.

\item  For $T = T_P$ (product), by applying  (\ref{eq:P1}) to $\delta^T_{ijk_1} > 0$, we obtain:
\[ a_{ij} > 0, \quad a_{k_1 j} > b_{k_1}, \quad \dfrac{b_i}{a_{ij}}   > \dfrac{b_{k_1}}{a_{k_1j}}.  \]
The inequality $a_{k_1 j} > b_{k_1}$ allows us to apply (\ref{eq:P1}) to $\delta^T_{k_1jk_2} > 0$ and to obtain:
\[ a_{k_1 j} > 0, \quad a_{k_2 j} > b_{k_2}, \quad \dfrac{b_{k_1}}{a_{k_1 j}}   > \dfrac{b_{k_2}}{a_{k_2j}}.  \]

So $k_1 \not=k_2$. Furthermore, we deduce from the previous inequalities:
\[a_{ij} > 0, \quad a_{k_2 j} > b_{k_2}, \quad 
\dfrac{b_i}{a_{ij}}   > \dfrac{b_{k_1}}{a_{k_1j}}> \dfrac{b_{k_2}}{a_{k_2j}}.\]
By applying (\ref{eq:P1}) (we   also have $b_i \leq a_{ij}$), we obtain 
$\delta^T_{ijk_2} > 0$. As, by hypothesis, we have 
$\max_{k\in R} \delta^T_{ijk} = 0$,  we necessarily have $k_2\in \overline{R}$.
\item For $T = T_L$ (Łukasiewicz’s t-norm), by applying (\ref{eq:L1}) to  $\delta^T_{ijk_1} > 0$, we obtain:
\[ b_i > 0, \quad b_i - a_{ij} > b_{k_1} -  a_{k_1 j}. 
\]
The inequality $b_{k_1} -  a_{k_1 j} < b_i - a_{ij} \leq 0 $ allows us to apply (\ref{eq:L1}) to $\delta^T_{k_1jk_2} > 0$ and we obtain:
\[ b_{k_1} > 0, \quad   b_{k_1} -  a_{k_1 j} 
> b_{k_2} -  a_{k_2 j}.\]
So $k_1 \not=k_2$.  Furthermore, we deduce from the previous inequalities:
\[ b_i > 0, \quad b_i - a_{ij} > b_{k_1} -  a_{k_1 j} > b_{k_2} -  a_{k_2 j}.\]
By applying (\ref{eq:L1}), we obtain 
$\delta^T_{ijk_2} > 0$. As by hypothesis, we have 
$\max_{k\in R} \delta^T_{ijk} = 0$, we necessarily have $k_2\in \overline{R}$.
\end{itemize}

\end{proof}
The following statement will be very useful to prove our main result:
\begin{proposition}\label{prop:taree}
We assume that $N_c \subseteq R$. Let  $i \in R \cap \overline{N_c}$ and $j \in M$ satisfy (\ref{eq:nH1}). For any integer $p \geq 1$, we can find a set of  indexes $\{k_1 , k_2 \dots , k_p\} \subseteq \overline{R}$ verifying: 
\begin{equation}\label{eq:infini1}\text{card}\,(\{ k_1 , k_2 \dots , k_p \}) = p\,\, \text{ and } \,\, 
\delta^T_{ijk_{1}} > 0,\,\,\delta^T_{ijk_{p}} > 0, \text{ and  for all }\,\,l \in \{1,2,\dots,p-1\},\,\, \delta^T_{k_{l}jk_{l+1}} > 0.
\end{equation}

\end{proposition}
\begin{proof}
We will prove this result by induction on $p$. \\For $p = 1$, this follows immediately from (Lemma \ref{lemma:5}). \\
For $p = 2$, let us take an index $k_1 \in \overline{R}$ satisfying $\delta_{ijk_{1}}^T > 0$ and as $\overline{R} \subseteq \overline{N_c}$, we deduce, by applying (Lemma \ref{Lemma:6}), an index $k_2 \in \overline{R}$ such that the set $\{k_1,k_2\}$ satisfy  (\ref{eq:infini1}). 

Assume that  we have constructed a set $\{k_1,k_2,\dots,k_{p-1}\} \subseteq \overline{R}$ such that: $$\text{card}(\{k_1,k_2,\dots,k_{p-1}\}) = p-1 \text{ and } \delta^T_{ijk_{1}} > 0,\,\,\delta^T_{ijk_{p-1}} > 0, \text{ and  for all }\,\,l \in \{1,2,\dots,p-2\},\,\, \delta^T_{k_{l}jk_{l+1}} > 0.$$
Let us prove the existence of an index $k_p$ such that the set $\{k_1,k_2,\dots,k_{p-1},k_p\}$ satisfy (\ref{eq:infini1}). We have to prove: \[k_p \in \overline{R}, \,\,
\text{card}\,(\{ k_1 , k_2 \dots , k_p \}) = p\,\, \text{ and } \,\,\delta_{k_{p-1}jk_p}^T > 0 \,\, \text{ and } \,\,  \delta_{ijk_p}^T > 0.\] 
As we have $k_{p-1} \in \overline{R}$ and $\overline{R} \subseteq \overline{N_c}$ we have: 
\[ \delta_{k_{p-1}} = \min_{l \in M} \max_{h \in N} \delta_{k_{p-1}lh}^T > 0.\]
For $l = j$, we deduce:
\[ \max_{h \in N} \delta_{k_{p-1}jh}^T > 0.\]
Let us take an index $k_p \in N$ such that $\delta_{k_{p-1}jk_p}^T > 0$. \\
Let us remark that if we prove that $\delta_{ijk_p}^T > 0$, by (\ref{eq:nH1}),  we conclude that $k_p \in \overline{R}$.\\
Let us check for each of the three t-norms the remaining conditions: $\text{card}\,(\{ k_1 , k_2 \dots , k_p \}) = p$ and  $\delta_{ijk_p}^T > 0$. 
\begin{itemize}
    \item For $T = T_M$ (min), by the recurrence hypothesis, we apply $p -1$ times the equivalence  (\ref{eq:M1}) and we obtain:
    \begin{equation}\label{eq:prov:bibk}
    b_i > b_{k_1} > b_{k_2} > \dots > b_{k_{p-1}} \text{ and } a_{k_{p-1}j} > b_{k_{p-1}}. \end{equation}
    As we have $b_{k_{p-1}} < a_{k_{p-1}j}$, we can apply (\ref{eq:M1}) to $\delta_{k_{p-1}jk_p}^{T_M} > 0$ and we obtain: 
    \begin{equation}
        b_{k_{p-1}} > b_{k_{p}} \text{ and } a_{k_{p}j} >b_{k_{p}}.  
    \end{equation}
    Then, from (\ref{eq:prov:bibk}) we deduce $b_i > b_{k_1} > b_{k_2} > \dots > b_{k_{p-1}}> b_{k_{p}}$ and $a_{k_{p}j} >b_{k_{p}}$.\\
    So $\text{card}\,(\{ k_1 , k_2 \dots , k_p \}) = p$.  
    As we have $b_i \leq a_{ij}$ (Lemma \ref{lemma:5}), we conclude that $\delta_{ijk_p}^{T_M} > 0$.

    \item For $T = T_P$ (product),  by the recurrence hypothesis, we apply $p - 1$ times the equivalence  (\ref{eq:P1}) and we obtain:
    \begin{equation}\label{eq:prov:bibkTP}
    \frac{b_i}{a_{ij}} > \frac{b_{k_1}}{a_{k_1j}}  > \dots > \frac{b_{k_{p-1}}}{a_{k_{p-1}j}}  \text{ and } a_{ij} > 0, \text{ for all } l \in \{1,2,\dots,p-2\}, a_{k_lj} > 0, \text{ and }  a_{k_{p-1}j} > b_{k_{p-1}}. \end{equation}
    As we have $b_{k_{p-1}} < a_{k_{p-1}j}$, we can apply (\ref{eq:P1}) to $\delta_{k_{p-1}jk_p}^{T_P} > 0$ and we obtain: 
    \begin{equation}
        \frac{b_{k_{p-1}}}{a_{k_{p-1}j}} > \frac{b_{k_{p}}}{a_{k_{p}j}} \text{ and } a_{k_{p}j} >b_{k_{p}}.  
    \end{equation}
    Then, from (\ref{eq:prov:bibkTP}) we deduce $\frac{b_i}{a_{ij}} > \frac{b_{k_1}}{a_{k_1j}} > \frac{b_{k_2}}{a_{k_2j}} > \dots > \frac{b_{k_{p-1}}}{a_{k_{p-1}j}}> \frac{b_{k_{p}}}{a_{k_{p}j}}$ and $a_{k_{p}j} >b_{k_{p}}$.\\
    So $\text{card}\,(\{ k_1 , k_2 \dots , k_p \}) = p$.  
    As we have $b_i \leq a_{ij}$ (Lemma \ref{lemma:5}), we conclude that $\delta_{ijk_p}^{T_P} > 0$. 
    \item For $T = T_L$ (Łukasiewicz’s t-norm), by the recurrence hypothesis, we apply $p -1$ times the equivalence  (\ref{eq:L1}) and we obtain:
    \begin{equation}\label{eq:prov:bibkTL}
    {b_i}-{a_{ij}} > {b_{k_1}}-{a_{k_1j}}  > \dots > {b_{k_{p-1}}}-{a_{k_{p-1}j}}  \text{ and } b_i > 0, \text{ for all } l \in \{1,2,\dots,p-2\}, b_{k_l} > 0. \end{equation}
    As we have $b_{k_{p-1}}-{a_{k_{p-1}j}} < {b_i}-{a_{ij}} \leq 0$, we can apply (\ref{eq:L1}) to $\delta_{k_{p-1}jk_p}^{T_L} > 0$ and we obtain: 
    \begin{equation}
        {b_{k_{p-1}}}-{a_{k_{p-1}j}} > {b_{k_{p}}}-{a_{k_{p}j}}.  
    \end{equation}
    Then, from (\ref{eq:prov:bibkTL}) we deduce ${b_i}-{a_{ij}} > {b_{k_1}}-{a_{k_1j}} > {b_{k_2}}-{a_{k_2j}} > \dots > {b_{k_{p-1}}}-{a_{k_{p-1}j}}> {b_{k_{p}}}-{a_{k_{p}j}}$ and $b_i > 0$. \\So $\text{card}\,(\{ k_1 , k_2 \dots , k_p \}) = p$ and  by (\ref{eq:L1}) we have $\delta_{ijk_p}^{T_L} > 0$. \\
\end{itemize}
\end{proof}

We now prove our main result:
\begin{theorem}\label{th:1}
Assume that there is at least one equation of index $i$ that is solvable independently of the others, i.e.,  there is a subsystem reduced to a single equation $(S_{\{ i \}})$  that is consistent. 
Then, we have:
\begin{enumerate}
\item The set $N_c$ is non-empty ($N_c \not= \emptyset$).
\item The system $(S_{N_c}) : A_{N_c} \Box_T^{\max} x =  b_{N_c}$  is consistent.
\item 
Among the consistent subsystems 
$(S_R) : A_R \Box_T^{\max} x =  b_R$ of the system $(S) : A \Box_T^{\max} x = b$ with $R \subseteq \{1,2,\dots,n\}$,  
the consistent subsystem $(S_{N_c}) : A_{N_c} \Box_T^{\max} x =  b_{N_c}$  is maximal in the following sense:   any subsystem defined by a strict superset of $N_c$ is inconsistent. 
\end{enumerate}
\end{theorem}
The third statement of the above theorem means that for any non-empty subset $R \subseteq \{1 , 2 , \dots  , n\}$, we have:
\[ \text{ The system }  (S_R) \text{ is consistent and }  N_c \subseteq R  
\Longrightarrow 
N_c = R.\]
We begin by proving the second statement, and we will then prove the first and third statements {\it simultaneously}.
\begin{proof}\mbox{}\\
$\bullet$ Proof of the second  statement (assuming that $N_c \not= \emptyset$).

To prove that the subsystem $(S_{N_c})$
is consistent, we must show that 
$\Delta_{N_c} = 0$, see (Lemma \ref{Lemma:subsys}),  which means that for any $i\in N_c$, there is an index $j\in M$ such that $\max_{k\in N_c} \delta^T_{ijk} = 0$.

Let an index $i\in N_c$. We have:
\[   \delta_i =  \min_{j\in M} \max_{k\in N} \delta^T_{ijk} = 0.\]
We can easily deduce that there is an index $j \in M$ such that:
\[0 = \max_{k\in N} \delta^T_{ijk} \geq  \max_{k\in N_c} \delta^T_{ijk}. \]

$\bullet$ Proof of the first and third  statement.

Let $R \subseteq N$ be a non-empty set  such that  $\text{the system } (S_R) \text{ is consistent and } N_c \subseteq R$. We will show  the equality  $N_c = R$.

{\it In particular, if we take $R = \{ i \}$ and we suppose that the subsystem $(S_R)$ is consistent, then    we trivially have $\emptyset \subset R$, thus $N_c \not=\emptyset$.}

Let us show that $N_c = R$  by contradiction.

Assume that we have $N_c \subset R$. Let an index $i\in R \cap \overline{N_c}$. 

By (Lemma \ref{lemma:4}), we can take an index $j\in M$ verifying  (\ref{eq:nH1}):
\[   \max_{k\in R} \delta^T_{ijk} = 0.\]

Then by (Proposition \ref{prop:taree}),  we conclude that for any integer $p \geq 1$, we have $p \leq \text{card}({\overline{R}})$, which is a contradiction.

We have proven that $N_c = R$.
\end{proof}

Consequently, the computational complexity for obtaining the maximal consistent subsystem $(S_{N_c})$ of an inconsistent system $(S)$ is the same as that for computing the Chebyshev distance $\Delta$.

\begin{example}
(continued)  We reuse the matrix $A$ and the vector $b$ of (Example \ref{ex:1}), see (\ref{eq:Abofpedrycz}).

\begin{itemize}
    \item For the $\max-\min$ system $A \Box_{\min}^{\max} x = b$, its unique maximal consistent subsystem is the system itself, since the system $A \Box_{\min}^{\max} x = b$ is consistent and therefore $N_c = \{1,2,3,4\}$.

    \item For the inconsistent max-product system $A \Box_{T_P}^{\max} x = b$ and the inconsistent max-Łukasiewicz system $A \Box_{T_L}^{\max} x = b$, we have $N_c = \{1,3,4\}$. So the system $(S_{N_c}) : A_{N_c} \Box_{T_P}^{\max} x = b_{N_c}$ is a  maximal consistent subsystem of the system $A \Box_{T_P}^{\max} x = b$ and the system $(S_{N_c}): A_{N_c} \Box_{T_L}^{\max} x = b_{N_c}$ is a maximal consistent subsystem of the system $A \Box_{T_L}^{\max} x = b$. For each inconsistent system, we obtain a consistent subsystem of it by removing the equation whose index is $2$.  
\end{itemize}
\end{example}

\section{\texorpdfstring{Method for easily finding all consistent subsystems of an inconsistent $\max-\min$ system}{Method for easily finding all consistent subsystems of an inconsistent max-min system}}
\label{sec:findingminmaxCS}

In this section, we give an efficient method for obtaining all consistent subsystems of an inconsistent $\max-\min$ system. We begin by giving some notations and establishing  (Lemma \ref{lemma:deltaRchapeau}). Based on this result, we introduce our method (Subsection \ref{subsec:method}) which allows us to obtain, iteratively, all  the maximal consistent subsystems of an inconsistent system (Proposition \ref{prop:mcsprop}). Finally, we illustrate our method by an interesting example.

\subsection{Notations}

Let $(S) : A \Box_{\min}^{\max} x = b$ be a $\max-\min$ system where $A = [a_{ij}]_{{}1 \leq i \leq n,1 \leq j \leq m} $ and $b = [b_{i}]_{{}1 \leq i \leq n}$. We suppose that the coefficients of $b$ are ordered in ascending order, i.e., we have:
\begin{equation}\label{eq:reorder}
    b_1 \leq b_2 \leq \dots \leq b_n.
\end{equation}

We reuse (Notation \ref{notations:prelem}) i.e., 
for any subset $R \subseteq \{1,2,\dots, n \}$, we form the following $\max-\min$ subsystem: $(S_R) : A_R \Box_{\min}^{\max} x =  b_R,$  where 
$A_R =  [a_{ij}]_{i\in R, 1 \leq j \leq m} \text{ and }   b_R = [b_{i}]_{i\in R}$.

We compute the Chebyshev distance associated to the subsystem $(S_R)$ by (see (Lemma \ref{Lemma:subsys})): 
\begin{equation}\label{eq:deltaR2}
    \Delta_R = \max_{i \in R} \delta^R_i \quad \text{ where } \quad \delta_i^R = \min_{1 \leq j \leq m} \max_{k \in R} \delta^{T_M}_{ijk}.
\end{equation} 
We assume that each of the equations in the system $(S)$ is solvable independently of the others, which means that all subsystems reduced to one equation are consistent:
\begin{equation}\label{eq:alluniqeqcons}
    \text{for all } i \in N,  \text{ the system } (S_{\{i\}}) : A_{\{i\}} \Box_{\min}^{\max} x = b_i \text{ is consistent.}
\end{equation} 
For all $s \in N$, we associate to the subsystem $(S_{\{1,2,\dots,s\}})$, the following set:
\begin{equation}\label{Plus1}
\mathcal{E}^s = \bigg\{R \in 2^{\{1 , 2 , \dots  , s\}}\,\mid \,
(S_R) : A_R \Box_{\min}^{\max} x = b_R \quad \text{is consistent}\bigg\}.
\end{equation}

The set $\mathcal{E}^s$ is non-empty since, from (\ref{eq:alluniqeqcons}), all singletons $\{1\}, \{2\}, \cdots, \{ s\}$ are in $\mathcal{E} ^s$.

\subsection{Preliminary result}  

The following result allows us  to construct, from a consistent subsystem of $(S_{\{1,2,\dots,s\}})$ a consistent subsystem with a larger number of equations, if it is possible.\\  

\begin{lemma}\label{lemma:deltaRchapeau}
Let $1 \leq s   < n - 1$ be fixed and consider a subset of equation indexes $R\in  {\cal E}^s$ which allows us to form a consistent subsystem of $(S_{\{1,2,\dots,s\}})$. For any $k\in N$ such that $k > s$, we put the following subset $\widehat R_k =  R \cup \{ k\}$ and we have (using the formula (\ref{eq:deltaR2})):
\begin{enumerate}
\item $\forall i \in R \,,\, \delta^{\widehat R_k}_i = 0$, 
\item The subsystem  $(S_{\widehat R_k})$ is a   consistent subsystem
$\Longleftrightarrow
\delta^{\widehat R_k}_{k} = 0.
$
\end{enumerate}
\end{lemma} 

\begin{proof}
It suffices to prove the first statement.

Let $i\in R$, we have $i < k$, so  
$b_i \leq b_{k}$ and  for all $l\in M$, we have:
\[\sigma_G(b_i, a_{k
l}, b_{k}) = \min(\dfrac{(b_i - b_k)^+}{2} , (a_{kl} - b_k)^+) = 0. \]
Since $\Delta_R = 0$, let $j\in M$  be such that  $\max_{k'\in R} \delta^{T_M}_{ijk'} = 0$. \\
From $k' = i \in R$ and (Lemma \ref{lemma:biaijplus}), we deduce $0 = \delta^{T_M}_{iji} = (b_i - a_{ij})^+$.\\ From (\ref{eq:deltaijkT}) and   the inequality $b_i \leq b_k$, we obtain:
\[ 
 \delta^{T_M}_{ijk} =  \sigma_G(b_i , a_{k j} , b_{k}) = 0.\]
Finally, we obtain from (\ref{eq:deltaR2}):
 \begin{align}
\delta^{\widehat R_k}_i & \leq \max_{k'\in \widehat R_k} \delta^{T_M}_{ijk'} \nonumber\\
& = \max(\max_{k'\in R} \delta^{T_M}_{ijk'}, \delta^{T_M}_{ijk})\nonumber\\
 & = 0.\nonumber\\
 \end{align}  
\end{proof}

\subsection{\texorpdfstring{Method for finding all consistent subsystems of an inconsistent $\max-\min$ system}{Method for finding all consistent subsystems of an inconsistent max-min system}}
\label{subsec:method}

We will now introduce our method, which allows us to build for $1 \leq s \leq n - 1$ the set $\mathcal{E}^{s+1}$ from the set $\mathcal{E}^{s}$. 
Trivially, we have $\mathcal{E}^1 = \bigg\{\{1\}\bigg\}$. \\For $1 \leq s \leq n - 1$ the set $\mathcal{E}^{s+1}$ is given by:
\begin{equation}\label{eq:esplus1toes}
     {\cal E}^{s + 1} = {\cal E}^{s} \cup \bigg\{ \{ s + 1 \} \bigg\} \cup \bigg\{  \widehat R  \mid  R  \in {\cal E}^{s}, \widehat R = R \cup \{ s + 1 \},  \text{ and } \delta^{\widehat R}_{s+1} = 0\bigg\}.
\end{equation}
The equality (\ref{eq:esplus1toes}) is easily deduced from (Lemma \ref{lemma:deltaRchapeau}).\\ 
Algorithmically, the main step to deduce ${\cal E}^{s + 1}$  from  ${\cal E}^{s}$ consists in   computing the $\text{card}({\cal E}^{s})$ numbers   $\delta^{\widehat R}_{s+1} = \min_{1 \leq j \leq m} \max_{k \in \widehat R} \delta^{T_M}_{s+1jk}$. From the set $\mathcal{E}^n$ we  get  all the consistent subsystems of the  system $(S)$.

\noindent
Starting from a consistent subsystem $(S_R)$ of the system $(S_{\{1 ,   \dots , s\}})$ that {\it we suppose to be maximal}  among the consistent subsystems   of the system $(S_{\{1 ,   \dots , s\}})$, we characterize in which case 
$(S_R)$ is a maximal consistent subsystem  of the whole system  $(S)$.
\begin{proposition}\label{prop:mcsprop}
Let $1 \leq s   < n - 1$ and $R\subseteq \{1 , \dots , s\}$. We suppose that the subsystem  
$(S_R)$ is a maximal consistent subsystem of $(S_{\{1 ,   \dots , s\}})$. Then we have: 
\begin{equation}
(S_R) \,\, \text{is a  maximal consitent subsystem of the system} \,\, (S) 
\Longleftrightarrow 
 \forall k\in \{s+1 , \dots  , n\},\,\,  
 \delta^{\widehat R_k}_k > 0.
\end{equation}
\end{proposition}
\begin{proof}
The proof of the implication $\Longrightarrow$ follows directly from (Lemma \ref{lemma:deltaRchapeau}). In fact, for any $k \in \{s+1 , \dots  , n\}$ we have 
$R \subset \widehat R_k = R \cup \{ k \}$. The  maximality of the subsystem $(S_R)$ implies the non-consistency of the subsystem $(S_{\widehat R_k})$, thus 
$\delta^{\widehat R_k}_k > 0$.

\noindent
To prove the implication $\Longleftarrow$,  let $U \subseteq N$ such that the subsystem $(S_U)$ is consistent and $R \subseteq U$, we must prove the equality $R = U$.\\ 
By assumption,    if $U$ satisfies the inclusion $U \subseteq \{1 , 2 , \dots ,s\}$, then   we conclude that    $R = U$. 

Let us prove the inclusion $U \subseteq \{1 , 2 , \dots ,s\}$ by contradiction.

Suppose that there is an index $k \in \{ s+ 1 , \dots , n\} \cap U$. Then we have: 
\[  R \subset \widehat R_k \subseteq U, \quad 
\widehat R_k = R \cup \{ k\}.\]
From the consistency of the subsystem $(S_U)$, we deduce that the subsystem $(S_{\widehat R_k})$ is consistent. By applying 
(Lemma \ref{lemma:deltaRchapeau}), we obtain $\delta^{\widehat R_k}_k = 0$  and then a contradiction. 
\end{proof}

We can then obtain, iteratively, all the maximal consistent subsystems of the whole inconsistent system.  (Example \ref{ex:allconssub}) shows that two maximal consistent subsystems \textit{do not necessarily have the same cardinality}.

\begin{example}\label{ex:allconssub}
    Let the system $A \Box_{\min}^{\max} x = b$ be defined by:
\begin{equation*}\label{eq:Ab}
    A =\begin{bmatrix}
    0.98& 0.02& 0.10\\
    0.80& 0.31& 0.18\\
    0.78& 0.38& 0.26\\
    0.77& 0.20& 0.85\\
         \end{bmatrix} \text{ and } b = \begin{bmatrix}0.13\\0.28 \\0.54\\
         0.70\\      \end{bmatrix}.
\end{equation*}
We have $b_1 < b_2 < b_3 < b_4$ and each subsystem reduced to a single equation is consistent i.e., for all $i \in \{1,2,3,4\}$, $\text{ the system } (S_{\{i\}}) : A_{\{i\}} \Box_{\min}^{\max} x = b_i \text{ is consistent.}$

We start from $\mathcal{E}^1 = \bigg\{\{1\}\bigg\}$. We have:
\begin{itemize}
    \item $\mathcal{E}^2 = \bigg\{ \{1\}, \{2\}, \{1,2\} \bigg\}$,
    \item $\mathcal{E}^3 = \bigg\{ \{1\}, \{2\}, \{3\},  \{1,2\} \bigg\}$,
    \item $\mathcal{E}^4 = \bigg\{ \{1\}, \{2\}, \{3\}, \{4\},  \{1,2\}, \{1,4\}, \{2,4\}, \{3,4\}, \{1,2,4\}   \bigg\}$.
\end{itemize}

So, $\{ 3,4 \}$ and $\{ 1,2, 4 \}$ are the two subsets of indexes describing the maximal consistent subsystems of the inconsistent $\max-\min$ system.
\end{example}

Our method is much faster than checking, using (\ref{eq:consiste}), the consistency of each of the subsystems $(S_R)$ of $(S)$ one by one.

\section{Conclusion}
 In this article, we studied the inconsistency of systems of $\max-T$ fuzzy relational equations of the form $A \Box_{T}^{\max} x = b$, where $T$ is a t-norm among $\min$, product or Łukasiewicz's t-norm.  By computing the Chebyshev distance $\Delta = \inf_{c \in \mathcal{C}} \Vert b - c \Vert$
 associated to the second member $b$ of an inconsistent   $\max-T$ system where $\mathcal{C}$ is the set of second members of consistent systems defined with the same matrix $A$,   using the analytical formulas given in \cite{baaj2023chebyshev,baaj2023maxmin}, we showed in (Theorem \ref{th:1}) that we can directly find a canonical maximal consistent subsystem of this inconsistent  system. We showed that the computational complexity for obtaining this maximal consistent subsystem is the same as that required for computing $\Delta$. 
 
 For an inconsistent $\max-\min$ system, we introduced a method for efficiently  obtaining  all  the consistent subsystems of this system. The method is based on the formula of $\Delta$ for $\max-\min$ systems and allows us to obtain, iteratively, all the maximal consistent systems of the considered inconsistent system. 
For the moment, we have not attempted to adapt this method to the cases of $\max-$product and $\max-$Łukasiewicz compositions.

This work may be useful for solving inconsistency issues in $\max-T$ systems involved in $\max-T$ learning methods, such as the paradigm of \cite{baaj2023maxmin} for  $\max-\min$ learning weight matrices  according to  training data, learning the rule parameters for possibilistic rule-based systems \cite{baajpossslearning,baaj2023maxmin} or learning associate memories \cite{sussner2006implicative}. It can also be useful for applications based on $\max-T$ systems, e.g. spatial analysis \cite{di2011spatial}, diagnostic problems \cite{dubois1995fuzzy}.

\bibliographystyle{plainnat} 
\bibliography{ref} 

\begin{thebibliography}{16}
\providecommand{\natexlab}[1]{#1}
\providecommand{\url}[1]{\texttt{#1}}
\expandafter\ifx\csname urlstyle\endcsname\relax
  \providecommand{\doi}[1]{doi: #1}\else
  \providecommand{\doi}{doi: \begingroup \urlstyle{rm}\Url}\fi

\bibitem[Baaj(2022)]{baajpossslearning}
Isma{\"\i}l Baaj.
\newblock Learning rule parameters of possibilistic rule-based system.
\newblock In \emph{2022 IEEE International Conference on Fuzzy Systems
  (FUZZ-IEEE)}, pages 1--8. IEEE, 2022.

\bibitem[{Baaj}(2023{\natexlab{a}})]{baaj2023chebyshev}
Isma{\"\i}l {Baaj}.
\newblock {Chebyshev distances associated to the second members of systems of
  Max-product/Lukasiewicz Fuzzy relational equations}.
\newblock \emph{arXiv e-prints}, art. arXiv:2302.08554, January
  2023{\natexlab{a}}.
\newblock \doi{10.48550/arXiv.2302.08554}.

\bibitem[{Baaj}(2023{\natexlab{b}})]{baaj2023maxmin}
Isma{\"\i}l {Baaj}.
\newblock {Max-min Learning of Approximate Weight Matrices from Fuzzy Data}.
\newblock \emph{arXiv e-prints}, art. arXiv:2301.06141, January
  2023{\natexlab{b}}.
\newblock \doi{10.48550/arXiv.2301.06141}.

\bibitem[Cechl{\'a}rov{\'a} and Dikoxe(1999)]{cechlarova1999resolving}
Katar{\'\i}na Cechl{\'a}rov{\'a} and Pavel Dikoxe.
\newblock Resolving infeasibility in extremal algebras.
\newblock \emph{Linear algebra and its applications}, 290\penalty0
  (1-3):\penalty0 267--273, 1999.

\bibitem[Cuninghame-Green and
  Cechl{\'a}rov{\'a}(1995)]{cuninghame1995residuation}
RA~Cuninghame-Green and Katar{\'\i}na Cechl{\'a}rov{\'a}.
\newblock Residuation in fuzzy algebra and some applications.
\newblock \emph{Fuzzy Sets and Systems}, 71\penalty0 (2):\penalty0 227--239,
  1995.

\bibitem[Di~Martino and Sessa(2011)]{di2011spatial}
Ferdinando Di~Martino and Salvatore Sessa.
\newblock Spatial analysis and fuzzy relation equations.
\newblock \emph{Advances in Fuzzy Systems}, 2011:\penalty0 6--6, 2011.

\bibitem[Di~Nola et~al.(1984)Di~Nola, Pedrycz, Sessa, and Zhuang]{di1984fuzzy}
Antonio Di~Nola, Witold Pedrycz, Salvatore Sessa, and Wang~Pei Zhuang.
\newblock Fuzzy relation equation under a class of triangular norms: A survey
  and new results.
\newblock \emph{Stochastica}, 8\penalty0 (2):\penalty0 99--145, 1984.

\bibitem[Dubois and Prade(1995)]{dubois1995fuzzy}
Didier Dubois and Henri Prade.
\newblock Fuzzy relation equations and causal reasoning.
\newblock \emph{Fuzzy sets and systems}, 75\penalty0 (2):\penalty0 119--134,
  1995.

\bibitem[Klement et~al.(2013)Klement, Mesiar, and Pap]{klement2013triangular}
Erich~Peter Klement, Radko Mesiar, and Endre Pap.
\newblock \emph{Triangular norms}, volume~8.
\newblock Springer Science \& Business Media, 2013.

\bibitem[Li(2009)]{liPhD2009fuzzy}
Pingke Li.
\newblock \emph{Fuzzy Relational Equations: Resolution and Optimization}.
\newblock PhD thesis, North Carolina State University, 2009.

\bibitem[Miyakoshi and Shimbo(1985)]{miyakoshi1985solutions}
Masaaki Miyakoshi and Masaru Shimbo.
\newblock Solutions of composite fuzzy relational equations with triangular
  norms.
\newblock \emph{Fuzzy Sets and Systems}, 16\penalty0 (1):\penalty0 53--63,
  1985.

\bibitem[Pedrycz(1982)]{pedrycz1982fuzzy}
Witold Pedrycz.
\newblock Fuzzy relational equations with triangular norms and their
  resolutions.
\newblock \emph{Busefal}, 11:\penalty0 24--32, 1982.

\bibitem[Pedrycz(1985)]{pedrycz1985generalized}
Witold Pedrycz.
\newblock On generalized fuzzy relational equations and their applications.
\newblock \emph{Journal of mathematical Analysis and applications},
  107\penalty0 (2):\penalty0 520--536, 1985.

\bibitem[Sanchez(1976)]{sanchez1976resolution}
Elie Sanchez.
\newblock Resolution of composite fuzzy relation equations.
\newblock \emph{Information and control}, 30\penalty0 (1):\penalty0 38--48,
  1976.

\bibitem[Sanchez(1977)]{sanchez1977}
Elie Sanchez.
\newblock Solutions in composite fuzzy relation equations: Application to
  medical diagnosis in brouwerian logic.
\newblock In M.~M. Gupta, G.~N. Saridis, and B.~R. Gaines, editors, \emph{Fuzzy
  automata and decision processes}, pages 221--234. Amsterdam: North-Holland,
  1977.

\bibitem[Sussner and Valle(2006)]{sussner2006implicative}
Peter Sussner and Marcos~Eduardo Valle.
\newblock Implicative fuzzy associative memories.
\newblock \emph{IEEE Transactions on Fuzzy Systems}, 14\penalty0 (6):\penalty0
  793--807, 2006.

\end{thebibliography}
\end{document}